\documentclass{article}

\usepackage{tikz}
\usetikzlibrary{shapes.geometric}
\usetikzlibrary{positioning,automata,arrows}
\usepackage{pifont}

\usepackage{fontawesome,wasysym,marvosym}

\definecolor{UofTBlue}{RGB}{0,47,101}

\usepackage[final, nonatbib]{neurips_2020}
\usepackage[utf8]{inputenc} 
\usepackage[T1]{fontenc}    
\usepackage{hyperref}       
\usepackage{url}            
\usepackage{booktabs}       
\usepackage{amsfonts}       
\usepackage{nicefrac}       
\usepackage{microtype}      
\usepackage[ruled,linesnumbered]{algorithm2e}
\usepackage{amsthm}
\usepackage{amssymb}
\usepackage{color}
\usepackage{tikz}
\usepackage[font=small,labelfont=bf]{caption}
\newtheorem{theorem}{Theorem}
\newcommand{\entry}[2]{\raisebox{2pt}{\tikz{\draw[#1,line width=1.7pt] (0,0) -- (0.6,0);}} #2}

\title{Learning Symbolic Representations for Reinforcement Learning of Non-Markovian Behavior}

\author{%
Phillip J.K. Christoffersen, ~~Andrew C. Li, ~~Rodrigo Toro Icarte, ~~Sheila A. McIlraith \\
  University of Toronto \& Vector Institute for Artificial Intelligence, Toronto, Canada \\
  \texttt{\{phill,andrewli,rntoro,sheila\}@cs.toronto.edu}
}

\begin{document}

\maketitle

\begin{abstract}

    Many real-world reinforcement learning (RL) problems necessitate learning complex, temporally extended behavior that may only receive reward signal when the behavior is completed. If the reward-worthy behavior is known, it can be specified in terms of a non-Markovian reward function---a function that depends on aspects of the state-action history, rather than just the current state and action. Such reward functions yield sparse rewards, necessitating an inordinate number of experiences to find a policy that captures the reward-worthy pattern of behavior. 
    Recent work has leveraged Knowledge Representation (KR) to provide a symbolic abstraction of aspects of the state that summarize reward-relevant properties of the state-action history and support learning a Markovian decomposition of the problem in terms of an automaton over the KR. Providing such a decomposition has been shown to vastly improve learning rates, especially when coupled with algorithms that exploit automaton structure. Nevertheless, such techniques rely on \emph{a priori} knowledge of the KR. In this work, we explore how to automatically \emph{discover} useful state abstractions that support learning automata over the state-action history. The result is 
    an end-to-end algorithm that can learn optimal policies with significantly fewer environment samples than state-of-the-art RL on simple non-Markovian domains. 

\end{abstract}


\section{Introduction}

Deep RL has shown promise at learning complex behavior in many settings, including game-playing \cite{silver2018general}, robotics \cite{peters2003reinforcement}, and control systems \cite{crites1996improving}. These algorithms typically take advantage of a Markov assumption --- that it is enough to consider only the current state when deciding which action to take. However, many real-world tasks are inherently temporally extended. The pattern of behavior the RL agent must learn depends not only upon the current state but also on past states and actions. 
For example, an agent that needs to get through a locked door to get high reward must have previously located and acquired the key. Unfortunately, learning such temporally extended behavior can be incredibly challenging since the agent must learn to discern relevant features from its state-action history; these can be arbitrarily far removed from the present state, and may depend on this history in complex ways and without intermediate reward signal to aid learning. The standard deep RL solution to learning such temporally extended behavior is to use a recurrent neural network (RNN), which learns an abstract hidden state in order to summarize environment histories, but RNNs require much data to train, and are difficult to tune.


By contrast, in recent work, an algorithm for learning temporally extended behavior is proposed, where the RNN hidden states are replaced with an augmentation of the current state in terms of hand-designed propositional symbols, which incorporate domain knowledge and point the agent towards potentially reward-relevant properties of the state-action history (e.g. \cite{toroicarteWKVCM2019learning,camacho2019ltl,gaon2020reinforcement,icml2018rms}). In the previous example, one can augment the agent with the propositional symbol \textbf{have\_keys}, indicating whether the agent has acquired the keys to the door in the past. Markovian policies on this state space now become vastly more expressive: if one can additionally condition on the truth state of \textbf{have\_keys} when deciding an action, we can perform the following temporally extended behavior: go towards the keys when \textbf{have\_keys} is false, then go towards the door when \textbf{have\_keys} is true. But how did we know to augment the agent with \textbf{have\_keys} in the first place? While this example seems simple, this is only because we contrived the reward: it is not in general clear which propositional symbols to augment an agent with, in order to achieve high performance. To address this, we propose the use of automata learning within the RL framework to \emph{automatically} yield such propositional symbols, rather than relying on domain knowledge. We demonstrate that the trained automata dramatically accelerate policy learning, with our end-to-end approach outperforming a state-of-the-art RL algorithm (Recurrent-PPO) on several non-Markovian reward domains.

\section{Related Work}

Recently, the idea of specifying non-Markovian reward functions in RL via formal languages such as Linear Temporal Logic (LTL) 
(e.g., ( \cite{lacerda2014optimal,littman2017environment,camacho2017decision,li2018policy,toro2018teaching,brafman2018ltlf,camacho2019ltl,hasanbeig2019reinforcement}) or automata (e.g., \cite{icml2018rms,toroicarteWKVCM2019learning, brafman2019regular,hasanbeig2019logic}) has garnered significant attention. 
While these approaches rely on domain knowledge and a domain-specific vocabulary for specification of the reward function, we consider a black-box non-Markovian reward and present an automated approach to uncover the reward structure. In this approach, we train automata offline using a reward-prediction heuristic, and augment the environment states with the states of the learned automaton, as opposed to hand-designed features. An alternate black-box approach to ours is to first train an RNN, with a standard deep (recurrent) RL algorithm, and then "quantize" the hidden state of the RNN, but this learned transition model is not a direct function of the state-action history \cite{koul2019representations}.

Previous work by Toro Icarte et al. \cite{toroicarteWKVCM2019learning} and Xu et al. \cite{xu2020joint} share many of the motivations of our work but perform poorly in noisy environments. The work most similar to ours is by Gaon \& Brafman \cite{gaon2020reinforcement}, which we build on in several key ways. First, the (off-the-shelf) automata-learning approaches they employ are sensitive to noisy data and often learn large, sample-sensitive automata even when the reward structure is simple. We make use of recent advances in automata-learning which are robust to noise and regularize the size of the automaton, which we demonstrate in Section~\ref{sec:exp}. Furthermore, \cite{gaon2020reinforcement} lacked experimental comparisons against state-of-the-art RL. In our experiments with non-Markovian goals, we outperform state-of-the-art RL based on RNNs.

\section{Preliminaries}

An MDP is a tuple $\mathcal{M} = (S, A, P, R, \gamma)$, where $S$ is a set of states, $A$ a set of actions, $P: S \times A \times S \to [0, 1]$ the state-action transition function, $R: S \to \mathbb{R}$ the reward, and $0 \le \gamma \le 1$ the discounting factor \cite{sutton2018reinforcement}. In such a setup, the reward $R$ is considered Markovian, due to its dependence only on the most recent state. We will consider the following extension of the MDP: an NMRDP \cite{bacchus1996rewarding} (non-Markovian Reward Decision Process) $\mathcal{N} = (S, A, P, R, \gamma)$ is as before, but where the reward $R: \mathcal{H} \to \mathbb{R}$, where $\mathcal{H} = (S \times A)^*$ is the set of finite histories with states $S$ and actions $A$: in other words, the agent can be rewarded for behavior which is arbitrarily far removed from its current experience. Further, we define a proposition as a function $P: \mathcal{H} \to \{\mathrm{True}, \mathrm{False}\}$, in an NMRDP. Intuitively, propositions correspond to facts about the state-action history in a given episode, such as "the agent has at some point reached the top right corner" or "within the 3 most recent timesteps, the agent took action x". While the number of possible propositions grows double-exponentially in the length of the episode, domain knowledge is often used to specify a relevant set of such propositions, under which the reward is Markovian. The RL domains we consider have non-Markovian goals, i.e. there exists $\mathcal{G} \subset \mathcal{H}$ where $R(g) = 1$ for $g \in \mathcal{G}$, and $R(h) = 0$ for $h \in \mathcal{H} - \mathcal{G}$. Intuitively, we want to create a policy which makes the agent attain a goal history as soon as possible. 

\section{Algorithm}

We use the algorithm described in Algorithm~\ref{AutRL} named AutRL. Following each period of Markovian learning, the sampled traces are used to train (offline) a deterministic finite automaton (DFA) $M$ to predict whether a given sequence achieves reward $0$ or $1$. We leverage the DFA-learning approach from \cite{shvo2020interpretable} due to its efficiency, its propensity to learn small DFAs with few transitions, and its robustness to noise. Tabular $Q$-learning is used for \texttt{markov\_learn}. Intuitively, a DFA with state space $Q$ that accurately discriminates reward $1$ traces from reward $0$ traces (which we define as \emph{consistent}) must model all parts of the state-action history relevant to the goal, and therefore the augmented state space $S \times Q$ must make the problem Markovian. We remark that the resultant DFAs are functions $\mathcal{H} \rightarrow Q$ and are learned end-to-end without domain knowledge. Examples of this can be seen in Section~\ref{sec:exp}. We also provide a convergence guarantee for this algorithm in Appendix A.


\begin{algorithm}[t]
    dfa $\leftarrow $ empty\_automata\;
    $\pi$ $\leftarrow$ uniform\_random\_policy\;
    traces $\leftarrow \emptyset$ \; 
     \While{true}{
      sample\_traces $\leftarrow$ sample($\pi$, N) \;
      append traces with sample\_traces\;
      \If{\emph{sample\_traces} inconsistent with \emph{dfa}}{
        dfa $\leftarrow$ aut\_learn(traces)\;
      }
      $\pi$ $\leftarrow$ markov\_learn(sample\_traces $\times$ dfa)\;
     }
 \caption{AutRL}
 \label{AutRL}
\end{algorithm}


We note that our implementation using $Q$-learning converges to an optimal policy as the number of environment samples approaches infinity (assuming $\mathcal{G}$ is regular, as above) due to the optimal convergence guarantees of $Q$-learning on MDPs \cite{sutton2018reinforcement}). Note that while we search for consistent DFAs, this condition is not necessary to make the learning problem Markovian. For this reason, we relax the inconsistency condition analyzed above, replacing the DFA only under weak performance (i.e. low average reward) at the end of a given epoch of \texttt{markov\_learn}.





\section{Experimental Results}
\label{sec:exp}


The purpose of our experiments was to evaluate our AutRL algorithm, which leverages a learned symbolic representation, relative to a state-of-the-art RNN-based deep RL algorithm. The two metrics were the quality of the policies in terms of maximizing reward, and their sample efficiency. We tested on four non-Markovian domains, similar to those used in the experiments of \cite{gaon2020reinforcement} as follows. 


\textbf{Multi-Armed Bandit:} A single-state environment with two actions ($\mathrm{left}, \mathrm{right}$) and episodes of length $6$. A reward of 1 is obtained only if the 6 actions performed are precisely $\mathrm{left}, \mathrm{right}, \mathrm{right}, \mathrm{left}, \mathrm{right}, \mathrm{left}$ in that order. 

\textbf{Hallway:} A $1 \times 10$ grid, aligned left-to-right, with actions $\mathrm{left}, \mathrm{right}$, and with episodes of length 30. The agent spawns at a random location on the left half of the grid and must first reach the right-most square, and then reach the left-most square to obtain a reward of 1. Before reaching the right-most square, the optimal action in every state is $\mathrm{right}$, but after reaching the right-most square, the optimal action becomes $\mathrm{left}$ in every state. 

\textbf{Gridworld:} A $5 \times 5$ gridworld (with $x,y$-coordinates from $0$ to $4$) with an episode length of 20 and actions $\mathrm{up}, \mathrm{down}, \mathrm{left}, \mathrm{right}$. The agent must first reach $(4,0)$ and then the opposite corner $(0,4)$ to collect a reward of $1$. We also tested a noisy version of this environment where the reward was randomly withheld in $10\%$ of successful traces, and actions had a random outcome $10\%$ of the time. 

We compared AutRL against Recurrent-PPO (\cite{schulman2017proximal}) (with an LSTM policy) using the OpenAI Baselines implementation \cite{baselines}. As shown in all of the below environments, AutRL converges to high reward policies in fewer environment samples than Recurrent-PPO does. Notice that AutRL, despite being automata-based, outperforms Recurrent-PPO even on the Stochastic Gridworld, showing that AutRL is robust to noisy environments. 
In Figure \ref{fig:dfa_1} are two such examples of learned automata on two different runs of the Hallway environment.  

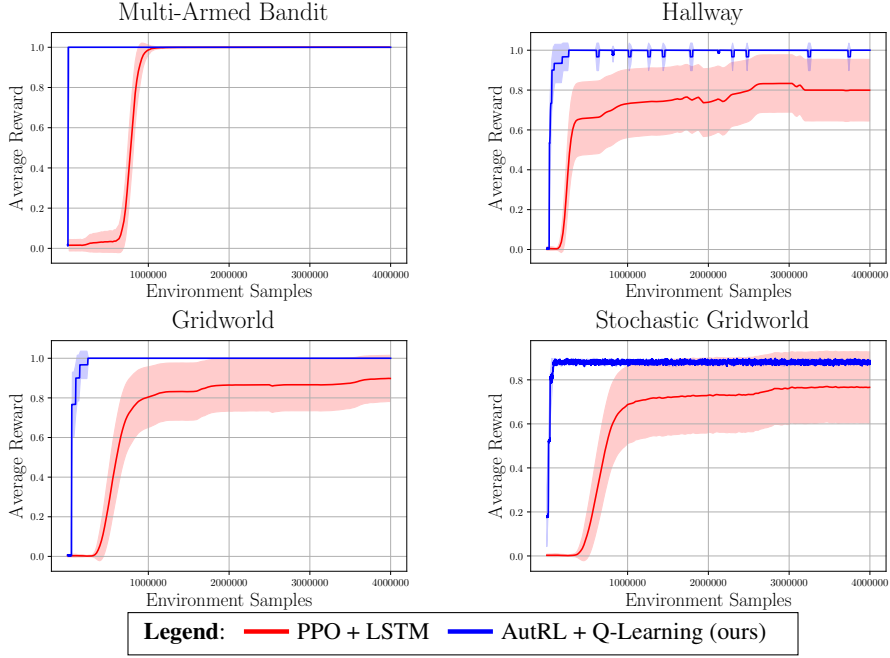
\begin{figure}[h!]
    \centering
     \includegraphics[scale=0.4]{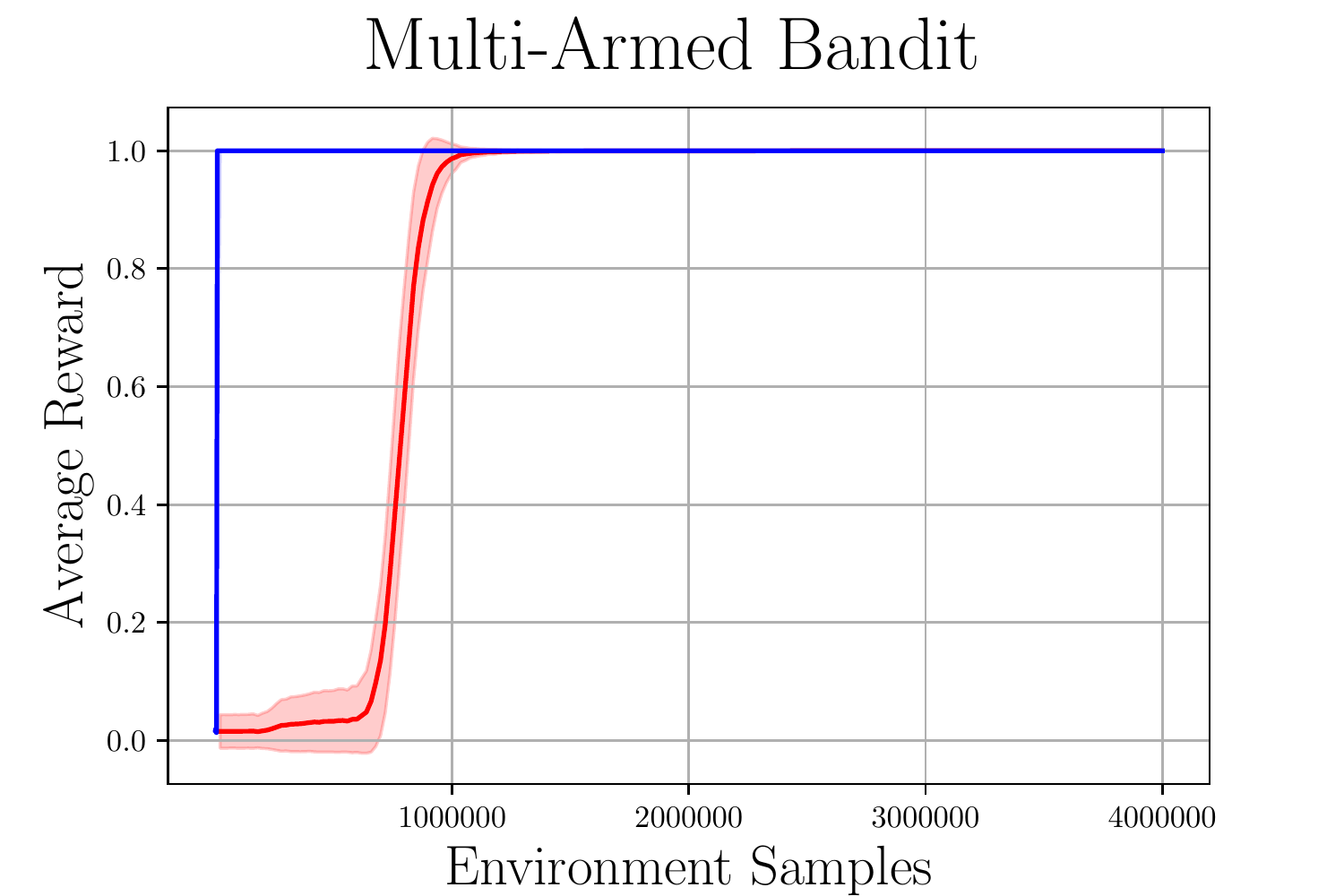} \hspace{1mm} \includegraphics[scale=0.4]{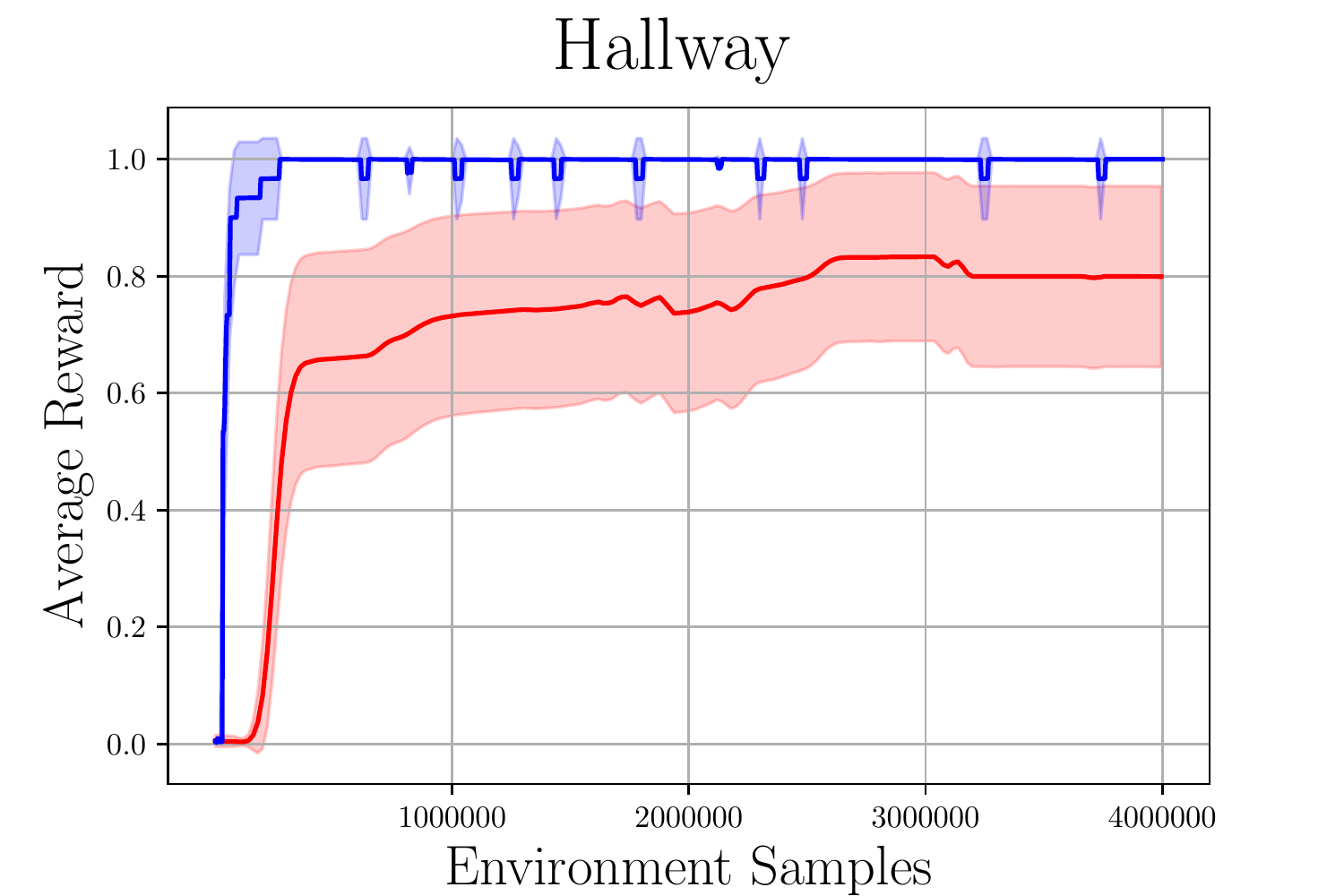} \hspace{1mm} \includegraphics[scale=0.4]{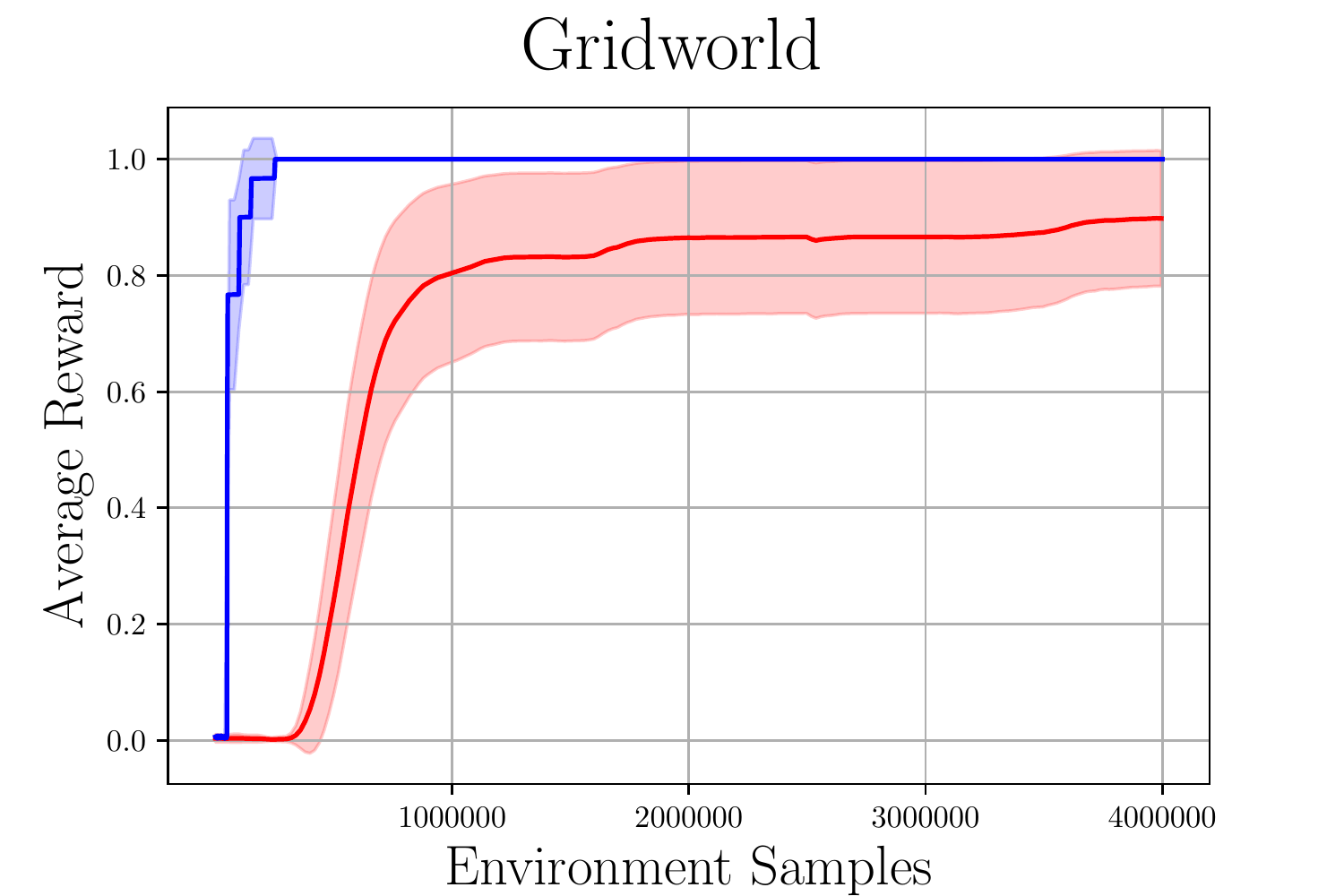} \hspace{1mm} \includegraphics[scale=0.4]{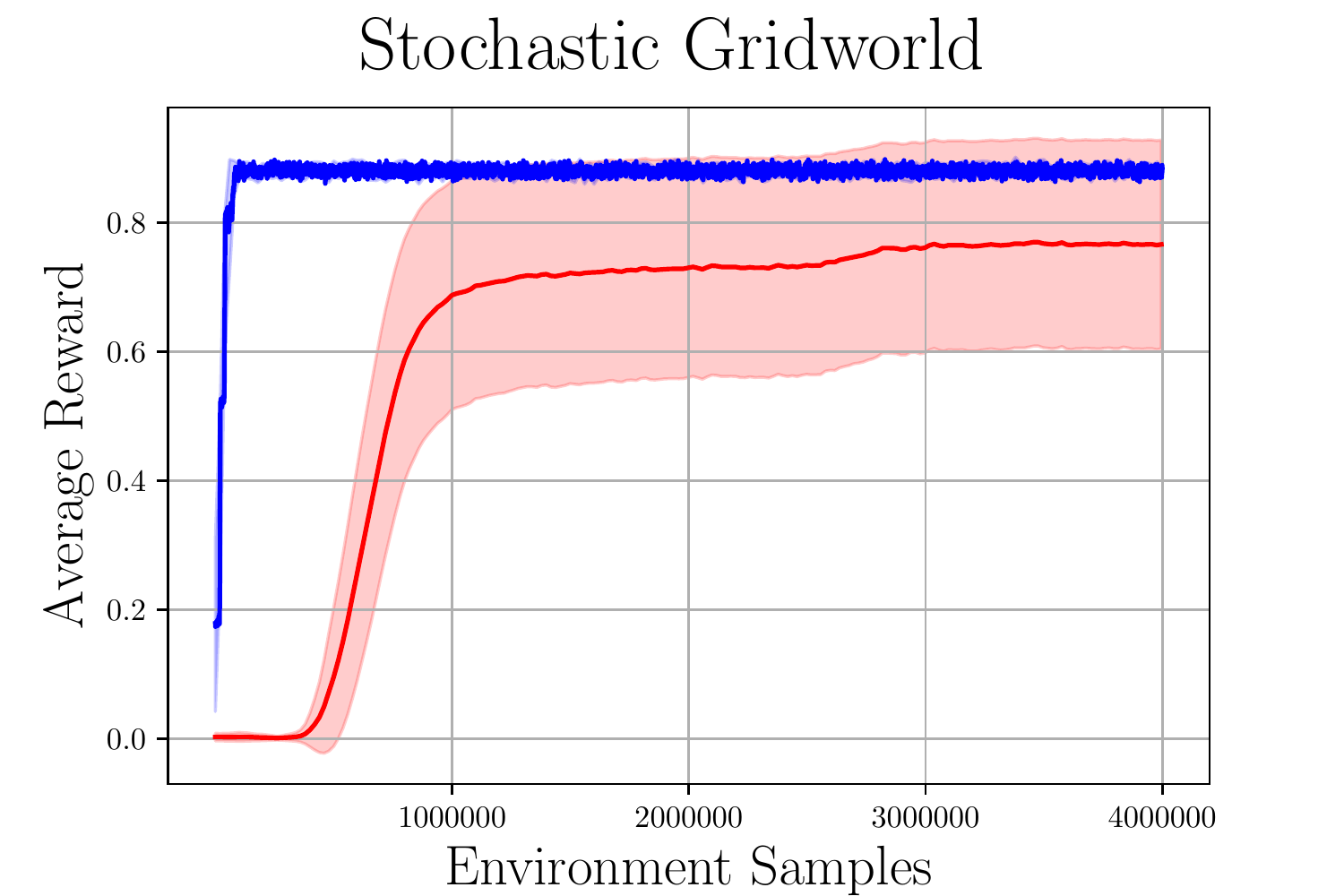} 
    {\renewcommand{\tabcolsep}{3pt}
\small\begin{tikzpicture}
\node[draw=black]  {%
	\begin{tabular}{llllllll}
	\textbf{Legend}:&
	\entry{red!}{PPO + LSTM}&
    \entry{blue!}{AutRL + Q-Learning (ours)} &
	\end{tabular}};
\end{tikzpicture}}%
    \caption{The results from the conducted experiments. The error bars are 95\% confidence intervals over 30 runs. AutRL learns a superior policy to the LSTM-based PPO in a way that is, at most, over order of magnitude more sample-efficient. Learning is far more stable across runs with AutRL: the error bars for AutRL are far tighter than those for PPO, indicating very little variation in the learning trajectory across runs. }
    \label{fig:results}
\end{figure}

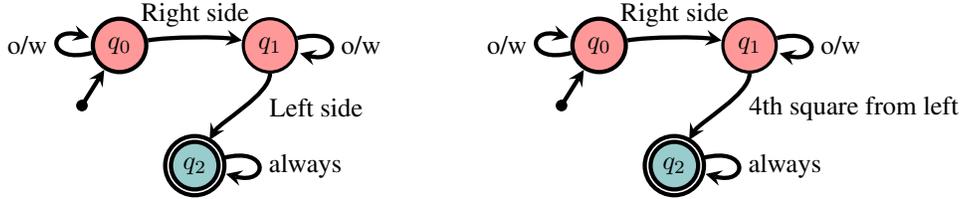
\begin{figure}[h!]
    \centering
    \vspace{-2mm}
    \begin{tikzpicture}[node distance=2cm,on grid,every initial by arrow/.style={ultra thick,->, >=stealth}, initial text={}]%
  \node[ultra thick,state,fill=red!40!white,minimum size=.5cm] (q_0) at (-1,0) {$q_0$};
  \node[ultra thick,state,accepting,fill=teal!40!white,minimum size=.5cm]  (q_2) at (0,-1.6) {$q_2$};
  \node[very thick,state,,fill=red!40!white,minimum size=.5cm]  (q_1) at (1,0) {$q_1$};

  \path[ultra thick,->, >=stealth] (-3/2,-0.8) edge node [right] {} (q_0);
  \draw[fill=black] (-3/2,-0.8) circle (0.07);
  
   \path[->] (q_0) edge [ultra thick,->, >=stealth,loop left] node { {o/w}} ();
   \path[->] (q_1) edge [ultra thick,->, >=stealth,loop right] node { {o/w}} ();
   \path[->] (q_2) edge [ultra thick,->, >=stealth,loop right] node { {always}} ();
  
  
  
  
   \draw[ultra thick,->, >=stealth, out=10, in=170, looseness=0.6] (q_0) to node [above] {Right side} (q_1);
   
   \draw[ultra thick,->, >=stealth, out=270, in=60, looseness=0.6] (q_1) to node [right=0.2] {Left side} (q_2);
  
  
  
  
  
  
  
  

\end{tikzpicture}%
 \hspace{10mm} \begin{tikzpicture}[node distance=2cm,on grid,every initial by arrow/.style={ultra thick,->, >=stealth}, initial text={}]%
  \node[ultra thick,state,fill=red!40!white,minimum size=.5cm] (q_0) at (-1,0) {$q_0$};
  \node[ultra thick,state,accepting,fill=teal!40!white,minimum size=.5cm]  (q_2) at (0,-1.6) {$q_2$};
  \node[very thick,state,,fill=red!40!white,minimum size=.5cm]  (q_1) at (1,0) {$q_1$};

  \path[ultra thick,->, >=stealth] (-3/2,-0.8) edge node [right] {} (q_0);
  \draw[fill=black] (-3/2,-0.8) circle (0.07);
  
   \path[->] (q_0) edge [ultra thick,->, >=stealth,loop left] node { {o/w}} ();
   \path[->] (q_1) edge [ultra thick,->, >=stealth,loop right] node { {o/w}} ();
   \path[->] (q_2) edge [ultra thick,->, >=stealth,loop right] node { {always}} ();
  
  
  
  
   \draw[ultra thick,->, >=stealth, out=10, in=170, looseness=0.6] (q_0) to node [above] {Right side} (q_1);
   
   \draw[ultra thick,->, >=stealth, out=270, in=60, looseness=0.6] (q_1) to node [right=0.2] {4th square from left} (q_2);
  
  
  
  
  
  
  
  

\end{tikzpicture}%
    \caption{Two distinct DFAs learned from the Hallway environment. Both are enough to make the domain Markovian: however, the one on the right is not a perfect classifier of $\mathcal{G}$ for this domain. This shows that the condition in the statement of AutRL above is strictly stronger than necessary: there are automata which do not perfectly decide $\mathcal{G}$, but which discern enough relevant information to learn optimal, temporally extended behavior in non-Markovian settings. }
    \label{fig:dfa_1}
\end{figure}

\section{Concluding Remarks}

In this work, we show how to learn a KR that provides a useful symbolic abstraction in support of realizing temporally extended behavior in standard RL agents which rely on a Markov assumption. We provide an end-to-end RL algorithm which, in deterministic and stochastic domains, is demonstrably more sample-efficient than the state of the art. We further provide theoretical guarantees of optimal convergence for our approach (under conditions outlined in the theorem of Appendix A), along with insight into the structure of DFAs which are learned by this method. 

This work reveals several promising directions for future research. While we considered environments where the reward function is non-Markovian, learning in the more general setting of POMDPs is an important problem related to this work, to which the DFA methods leveraged here do not apply. Further, AutRL leverages DFAs, which can only represent regular languages, thus learning efficiently when $\mathcal{G}$ is not a regular language in $S \times A$, as well as learning KR in high-dimensional and continuous state spaces, remain major open challenges to this methodology. 

%
 

\begin{ack}
We gratefully acknowledge funding from the Natural Sciences and
Engineering Research Council of Canada (NSERC), the Canada CIFAR AI Chairs
Program, and Microsoft Research. The third author also acknowledges
funding from ANID (Becas Chile). Resources used in preparing this research
were provided, in part, by the Province of Ontario, the Government of
Canada through CIFAR, and companies sponsoring the Vector Institute for
Artificial Intelligence (\url{www.vectorinstitute.ai/partners}). Finally, we
thank the Schwartz Reisman Institute for Technology and Society for
providing a rich multi-disciplinary research environment.
\end{ack}
\bibliographystyle{unsrt}
\bibliography{main}

\section*{Appendix A: Convergence Guarantee + Proof}

\begin{theorem}
    Let $\mathcal{N}$ be an NMRDP. Under the following (realistic) conditions, \texttt{\emph{AutRL}} converges to an optimal policy in the sample limit. 
    \begin{enumerate}
        \itemsep -0.25em 
        \item \texttt{\emph{sample\_traces}} visits every reachable history within $\mathcal{N}$ i.o.a.s. (infinitely often, almost surely) \textbf{(exploration)}
        \item For any regular language $\mathcal{L}$, when sampling traces using \texttt{\emph{sample\_traces}}, that \texttt{\emph{aut\_learn}} eventually returns an automaton perfectly deciding $\mathcal{L}$ w.p.1. \textbf{(consistency of \texttt{\emph{aut\_learn}})}
        \item \texttt{\emph{markov\_learn}} converges to the optimal policy in the sample limit for an MDP \textbf{(consistency of \texttt{\emph{markov\_learn}})}
        \item $\{h : R(h) = 1\} = \mathcal{G} \subset \mathcal{H}$ forms a regular language with alphabet $\Sigma = S \times A$ \textbf{(regularity)}
    \end{enumerate}
\end{theorem}

\begin{proof}
    Fix an NMRDP $\mathcal{N} = (S, A, P, R, \gamma)$, and assume all conditions are met in the theorem statement. Then, since $\mathcal{G}$ is a regular language, we have that the variable \textbf{dfa} is, with probability 1, eventually set to an automaton, with state space $Q$, which perfectly decides $\mathcal{G}$. At such a point, then we have that the states of the automaton perfectly predict the reward $R$: thus, the problem $\mathcal{M} = (S \times Q, A, P, R, \gamma)$, simply defined as $\mathcal{N}$ but with a state space augmented with the automaton states of \textbf{dfa}, is actually an MDP. Further, \textbf{dfa} is never thereafter reset, since no inconsistent trace can possibly be sampled from the environment subsequently. Thus, $\mathcal{M}$ is simply an MDP being trained with Markovian learning algorithm markov\_learn, and thus the appropriate convergence conditions apply.
\end{proof}

\section*{Appendix B: Baseline Hyperparameter Tuning}

For the tabular $Q$-learning for our approach (AutRL), we used a learning rate of $0.1$ on the deterministic environment, $0.001$ on the stochastic environment. Further, an exploration parameter of $0.01$ was used on the deterministic environments, and $0.05$ with a decaying exponential schedule (with factor $0.99$) was used for the stochastic environments. The automaton learning method was set at a maximum state threshold of $5$ for all environments but the multi-armed bandit (which had $14$), a loop penalty of $0.01$ in all environments, and a transition penalty of $0.01$ for the Multi-Armed Bandit, $0.3$ for both gridworlds, and $0.6$ for the hallway. Further, the timeout was set to $250$. 

For the Recurrent-PPO baseline, we used the OpenAI Baselines \cite{baselines} implementation with an LSTM policy network, an entropy coefficient of $0.001$, a learning rate of $0.0003$, a discounting factor of $0.99$, a batch size of $16384$ (with $8$ minibatches per update), and all other hyperparameters set to default. We followed standard practice in tuning these hyperparameters and found that these settings consistently performed best across our domains.

\end{document}